\documentclass{article}
\usepackage[a4paper, total={6in, 10in}]{geometry}

\usepackage[T1]{fontenc}
\usepackage{graphicx}
\usepackage[utf8]{inputenc} 
\usepackage[T1]{fontenc}    
\usepackage{hyperref}       
\usepackage{url}            
\usepackage{booktabs}       
\usepackage{amsfonts}       
\usepackage{nicefrac}       
\usepackage{microtype}      
\usepackage{xcolor}         

\usepackage{caption}
\usepackage{subcaption}

\usepackage{amsmath}
\usepackage{amsfonts}
\usepackage{amssymb}
\usepackage{amsthm}

\usepackage{pdflscape}
\usepackage{pdfpages}
\usepackage{longtable}

\usepackage{algorithm}
\usepackage{algpseudocode}
\usepackage{xypic}

\usepackage{paralist}
\usepackage[shortlabels]{enumitem}

\newcommand{\N}{\mathbb{N}}
\newcommand{\R}{\mathbb{R}}
\renewcommand{\P}{\mathbb{P}}
\newcommand{\D}{\mathcal{D}}
\renewcommand{\d}{\textnormal{d}}

\newcommand{\E}{\mathbb{E}}
\renewcommand{\d}{\textnormal{d}}

\renewcommand{\H}{\mathcal{H}}
\renewcommand{\1}{\mathbf{1}}

\newcommand{\MSE}{\textnormal{MSE}}
\newcommand{\Var}{\textnormal{Var}}

\newcommand{\loss}{\mathcal{L}}

\newcommand{\X}{\mathcal{X}}
\newcommand{\Y}{\mathcal{Y}}
\newcommand{\T}{\mathcal{T}}

\theoremstyle{definition}
\newtheorem{definition}{Definition}

\theoremstyle{plain}
\newtheorem{lemma}{Lemma}
\newtheorem{theorem}{Theorem}
\newtheorem{corollary}{Corollary}

\usepackage{apxproof}
\usepackage{cleveref}
\newcommand{\appproof}[2]{
\begin{proof}[Proof of \cref{#1}]
#2
\end{proof}
}

\begin{document}

\title{On the Change of Decision Boundaries and Loss in Learning with Concept Drift\thanks{We gratefully acknowledge funding by the BMBF TiM, grant number 05M20PBA.}}
\author{Fabian Hinder${}^{1}$, 
Valerie Vaquet${}^{1}$, \\
Johannes Brinkrolf${}^{1}$, and
Barbara Hammer${}^{1}$
\\\;\\
${}^{1}$CITEC, Bielefeld University, Bielefeld, Germany \\
\texttt{\{fhinder,vvaquet,jbrinkro,bhammer\}@techfak.uni-bielefeld.de}
}
%
%
%
\maketitle              
\begin{abstract}
The notion of concept drift refers to the phenomenon that the distribution generating the observed data changes over time. If drift is present, machine learning models may become inaccurate and need adjustment. Many technologies for learning with drift rely on the interleaved test-train error (ITTE) as a quantity which
approximates the model generalization error and triggers drift detection and model updates.
In this work, we investigate in how far this procedure is mathematically justified. More precisely, we relate a change of the ITTE  to 
 the presence of real drift, i.e.,\ a changed 
posterior, and to a change of the training result under the assumption of optimality.
We support our theoretical findings by  empirical evidence for  several learning algorithms, models, and datasets. 
\end{abstract}

\section{Introduction}
\label{sec:intro}

The world that surrounds us is subject to constant change, which also affects the increasing amount of data  collected over time,  in social media, sensor networks, IoT devices, etc.
Those changes, referred to as concept drift, can be caused by seasonal changes, changing demands of individual customers, aging or failing sensors, and many more. As drift constitutes a major issue in many applications, considerable research is focusing on this setting \cite{DBLP:journals/cim/DitzlerRAP15}.
Depending on the domain of data and application, different drift scenarios might occur:
For example, covariate shift refers to the situation that training and test sets have different marginal distributions~\cite{5376}. 

In recent years, a large variety of methods for learning in presence of drift has been proposed~\cite{DBLP:journals/cim/DitzlerRAP15},
whereby a majority of the approaches targets supervised learning scenarios.
Here, one distinguishes between virtual and real drift, i.e.\ non-stationarity of the marginal distribution only or also the posterior. Learning technologies often rely on windowing techniques and adapt the model based on the characteristics of the data in an observed time window. 
Here, many approaches use non-parametric methods or ensemble technologies~\cite{asurveyonconceptdriftadaption}.
Active methods explicitly detect drift, usually referring to drift of the classification error, and trigger model adaptation this way, while passive methods continuously adjust the model~\cite{DBLP:journals/cim/DitzlerRAP15}. Hybrid approaches combine both methods by continuously adjusting the model unless drift is detected and a new model is trained. 

In most techniques, evaluation takes place by means of the so-called 
interleaved train-test error (ITTE), which evaluates the current model on a given new data point before using it for training. This error is used to evaluate the overall performance of the algorithm, as well as to detect drifts in case of significant changes in the error or to control important parameters such as the window size~\cite{DBLP:journals/ijon/LosingHW18}.
Thereby,
these techniques often rely on
strong assumptions regarding the underlying process, e.g.,\ 
they detect a drift when the classification accuracy drops below a predefined threshold during a predefined time.
Such  methods 
face problems if the underlying drift characteristics do not align with these assumptions.

Here, we want to shed some light on the suitability of such choices and investigate 
the mathematical properties of the ITTE when used as an evaluation scheme. 
As the phenomenon of concept drift is 
widespread,
a theoretical understanding of the relation between
drift and the adaption behavior of learning models becomes crucial.
Currently, the majority of theoretical work for drift learning focuses on learning guarantees which are  similar in nature to the work of
 Vapnik in the batch case~\cite{mohri2012new,hanneke2015learning,hanneke2019statistical}.
Although those results provide interesting insights into the validity of learning models in the streaming setup, they 
focus on worst-case scenarios  and hence provide
very loose bounds on average only. 
In contrast, in this work, we focus on theoretical aspects of the learning algorithm itself in non-stationary environments, targeting general learning models including unsupervised ones. In contrast to the existing literature, we focus on alterations of models. This perspective is closely connected to the actual change of decision boundaries and average cases. In particular, we provide a mathematical substantiation of the suitability of the ITTE to evaluate model drift.

This paper is organized as follows: First (Section~\ref{sec:Setup}) we recall the basic notions of statistical learning theory and concept drift followed by reviewing the existing literature, positioning of this work with respect to it, and concertize the research questions (Section~\ref{sec:RelWork}). We proceed with a theoretical analysis focusing on (1) changes of the decision boundary in presence of drift (Section~\ref{sec:ModelDrift}), (2) changes of the training result (Section~\ref{sec:ChangeOfModel}), and (3) the connection of ITTE, drift, and the change of the optimal model (Section~\ref{sec:ChangeOfLoss}). 
Afterward, we empirically quantify the theoretical findings (Section~\ref{sec:Experiments})
and conclude with a summary (Section~\ref{sec:Conclusion}).

\section{Problem Setup, Notation, and Related Work}
\label{sec:Setup}

We make use of the formal framework for concept drift as introduced in \cite{DAWIDD,ContTime} as well as classical statistical learning theory, e.g., as presented in~\cite{shalev2014understanding}. In this section, we recall the basic notions of both subjects followed by a summary of the related work on learning theory in the context of concept drift.

\subsection{Basic Notions of Statistical Learning Theory}
\label{sec:IntroLT}

In classical learning theory, one considers a hypothesis class $\H$, e.g., a set of functions 
from $\R^d$ to $\R$, 
together with a non-negative loss function $\ell : \H \times (\X \times \Y) \to \R_{\geq 0}$ that is used to evaluate how well a model $h$ matches an observation $(x,y) \in \X \times \Y$ by assigning an error $\ell(h,(x,y))$. 
We will refer to $\X$ as the data space and $\Y$ as the label space. 
For a given distribution $\D$ on $\X \times \Y$
we consider $\X$- and $\Y$-valued random variables $X$ and $Y$, $(X,Y) \sim \D$, 
and assign the loss $\loss_{\D}(h) = \E[\ell(h,(X,Y))]$ to a model $h \in \H$. Using data sample $S \in \cup_{N \in \N} (\X \times \Y)^N$ consisting of i.i.d.\ random variables $S = ((X_1,Y_1),\dots,(X_n,Y_n))$ distributed according to $\D$, we can approximate $\loss_{\D}(h)$ using the empirical loss $\loss_S(h) = \frac{1}{n} \sum_{i = 1}^n \ell(h,(X_i,Y_i))$, which 
converges to $\loss_{\D}(h)$ almost surely. Popular loss functions are the mean squared error $\ell(h,(x,y)) = (h(x)-y)^2$, cross-entropy $\ell(h,(x,y)) = \sum_{i = 1}^n \1[y = i] \log h(x)_i$, 
or the 0-1-loss $\ell(h,(x,y)) = \1[h(x) \neq y]$. 
Notice that this setup also covers unsupervised learning problems
, i.e., $\Y = \{*\}$. 

In machine learning, training a model often refers to minimizing the loss $\loss_{\D}(h)$ using the empirical loss $\loss_S(h)$ as a proxy. A learning algorithm $A$, such as gradient descent schemes, selects a model $h$ given a sample $S$, i.e., $A : \cup_N (\X \times \Y)^N \to \H$. Classical learning theory investigates under which circumstances $A$ is consistent, that is, it selects a good model with high probability: $ \loss_\D(A(S)) \to \inf_{h^* \in \H} \loss_\D(h^*)$ as $|S| \to \infty$ in probability. 
loss $\loss_S$ and model $A(S)$ become dependent by training,
classical approaches aim for uniform bounds
$\sup_{h \in \H} |\loss_S(h) - \loss_\D(h)| \to 0$ as $|S| \to \infty$ in probability. 

\subsection{A Statistical Framework for Concept Drift}
\label{sec:IntroDrift}

The classical setup of learning theory 
assumes a time-invariant distribution $\D$ for all $(X_i,Y_i)$. 
This assumption is violated in many real-world applications, in particular, when learning on data streams. Therefore, we incorporate time into our considerations 
by means of an index set $\T$, representing time, and a collection of (possibly different) distributions $\D_t$ on $\X \times \Y$, indexed over $\T$~
\cite{asurveyonconceptdriftadaption}. In particular, the model $h$ and its loss also become time-dependent. It is possible to extend  this setup to a general statistical interdependence of data and time via a distribution $\D$ on $\T \times (\X \times \Y)$ which decomposes into a distribution $\P_T$ on $\T$ and the conditional distributions $\D_t$ on $\X \times \Y$  \cite{DAWIDD,ContTime}.
Notice that this setup  \cite{ContTime} is very general 
and can therefore be applied in different scenarios (see Section~\ref{sec:Conclusion}), 
albeit our main example is binary classification on a time interval, i.e.\ $\X = \R^d$, $\Y = \{0,1\}$, and $\T = [0,1]$. 

Drift refers to the fact that $\D_t$ varies for different time points, i.e.\ 
$\{ (t_0,t_1) \in \T^2 : \D_{t_0} \neq \D_{t_1} \}$
has measure larger zero w.r.t\ $\P_T^2$~\cite{DAWIDD,ContTime}. 
One further distinguishes a change of the posterior $\D_t(Y|X)$, referred to as \emph{real drift}, and of the marginal $\D_t(X)$, referred to as \emph{virtual drift}.
One of the key findings of \cite{DAWIDD,ContTime} is a unique characterization of the presence of drift by the property of  statistical dependency of time $T$ and data $(X,Y)$ if a time-enriched representation of the data $(T,X,Y) \sim \D$ is considered.
Determining whether or not there is drift during a time period is referred to as \emph{drift detection}. 

Since the distribution $\D_t$ can shift too rapidly to enable a faithful estimation of quantities thereof, we propose to address 
time windows $W \subset \T$ and to consider all data points, that are observed during $W$, analogous to an observation in classical learning theory. This leads to the following formalization~\cite{ida2022}:

\begin{definition}
Let $\X, \Y, \T$ be measurable spaces.
Let $(\D_t,\P_T)$ be a \emph{drift process}~\cite{DAWIDD,ContTime} on $\X \times \Y$ and $\T$, i.e.\ a distribution $\P_T$ on $\T$ and Markov kernels $\D_t$ from $\T$ to $\X \times \Y$.
A \emph{time window} $W \subset \T$ is a $\P_T$ non-null set. A \emph{sample (of size $n$) observed during $W$} is a tuple $S = ((X_1,Y_1),\dots,(X_n,Y_n))$ drawn i.i.d. from the mean distribution on $W$, that is $\D_W := \D(X,Y\mid T \in W)$.
\end{definition}

That resembles the practical procedure, where one obtains sample $S_1$ during $W_1$ from another sample $S_2$ during $W_2$, with $W_1 \subset W_2$, by selecting those entries of $S_2$ that are observed during $W_1$. 
In particular, if $W_1 = \{0,\dots,t-1\}, \; W_2 = \{0,\dots,t\}$ this corresponds to an incremental update.

In this work we will consider data drawn from a single drift process, thus we will  make use of the following short hand notation
$\loss_t(h) := \loss_{\D_t}(h)$ for a time point $t \in \T$ and $\loss_W(h) = \loss_{\D_W}(h)$ for a time window $W \subset \T$, where $\D_W = \E[\D_T\mid T \in W]$ denotes the mean of $\D_t$ during $W$ and $\loss(h) := \loss_\T(h)$ is the loss on the entire stream. Notice that this is well defined, i.e.,\ $\loss_W(h) = \E[\ell(h,(X,Y)) \mid T \in W] = \E[\loss_T(h) \mid T \in W]$ assuming $\loss(h) < \infty$. 
In stream learning, some algorithms put more weight on newer observations, e.g., by continuously updating the model. Such considerations can be easily integrated into our framework, but we omit them for simplicity.

\subsection{Related Work, Existing Methods, and Research Questions}
\label{sec:RelWork}

\begin{algorithm}[tb]
   \caption{Typical Stream Learning Algorithm}
   \label{alg:streamalg}
\begin{algorithmic}[1]
   \State {\bfseries Input:} {$S$ data stream, $A$ training algorithm, $h_0$ initial model, $\ell$ loss function, $D$ drift detector} 
   \State Initialize model $h \gets h_0$
   \While{Not at end of stream $S$}
   \State Receive new sample $(x,y)$ from stream $S$
   \State 
   Compute ITTE $L \gets \ell(h,(x,y))$
   \State Update model $h \gets A(h,(x,y))$ \Comment{Passive Adaption} \label{alg:streamalg:passiv}
   \If{Detect drift $D(L)$} 
   \State{Reset model $h \gets h_0$ OR Retrain on next samples}\Comment{Active Adaption}\label{alg:streamalg:active}
   \EndIf
   \EndWhile
\end{algorithmic}
\end{algorithm}

Algorithm~\ref{alg:streamalg} shows the outline of a typical (hybrid) stream learning algorithm.
Stream learning algorithms can be split into two categories~\cite{DBLP:journals/cim/DitzlerRAP15}: passive methods, which adapt the model slightly in every iteration (line~\ref{alg:streamalg:passiv}), and active methods, which train a new model once drift is detected (line~\ref{alg:streamalg:active}). There also exist hybrid methods that integrate both characteristics.

Most existing theoretical work on stream learning in the context of drift derives learning guarantees as inequalities of the following form: the risk on a current time window $W_2 \: (=\{t+1\})$ is bounded using the risk on a time window $W_1 \: (= \{1,\dots,t\})$ and a distributional difference in between those windows~\cite{mohri2012new,hanneke2019statistical}:
\begin{align}
    \underbrace{\loss_{W_2}(h)}_{\text{application time risk}} \leq\:\: \underbrace{\loss_{W_1}(h)}_{\text{train time risk}} \:\:+\quad \underbrace{\sup_{h' \in \mathcal{H}} \left| \loss_{W_2}(h') - \loss_{W_1}(h')\right|}_{\text{distributional discrepancy}}. \label{eq:loss_decompose}
\end{align}
Most approaches aim for a good upper bound of the train time risk~\cite{mohri2012new,hanneke2015learning}. 
This inequality then gives rise to convergence guarantees, which are usually applied by splitting the so far observed stream into several chunks and training a model on each of them~\cite{hanneke2019statistical,hanneke2015learning}.

A crucial aspect of the inequality is the distributional discrepancy.
Notice that it is closely related to other statistical quantities 
like the total variation norm~\cite{DAWIDD,mohri2012new,hanneke2019statistical} or the Wasserstein distance.  It 
provides a bound that refers to the worst possible outcome regarding the drift.
Although this scenario can theoretically occur, (see examples given in \cite[Theorem 2]{hanneke2015learning}), it is  not likely in practice. 
For example, for kernel based binary classifiers on $\Y = \{1,2\}$ it holds:
\begin{align*}
        \sup_{h \in \mathcal{H}} \left| \loss_{W_1}(h) - \loss_{W_2}(h)\right| 
    &=\text{MMD}\left( \frac{\sum_i \D_{W_i}(X, Y=i)}{\sum_i \D_{W_i}(Y=i)}, \frac{\sum_i \D_{W_i}(X, Y\neq i)}{\sum_i \D_{W_i}(Y\neq i)} \right),
\end{align*}
where MMD refers to the maximum mean discrepancy.
This term is closely related to the statistic used in popular unsupervised 
drift detectors~\cite{kernel2sampletest}. Thus, we obtain large values even if the decision boundary is not affected by drift. 

\begin{figure}[t]
    \centering
    \begin{minipage}[b]{0.31\textwidth}
    \centering
    \includegraphics[width=\textwidth]{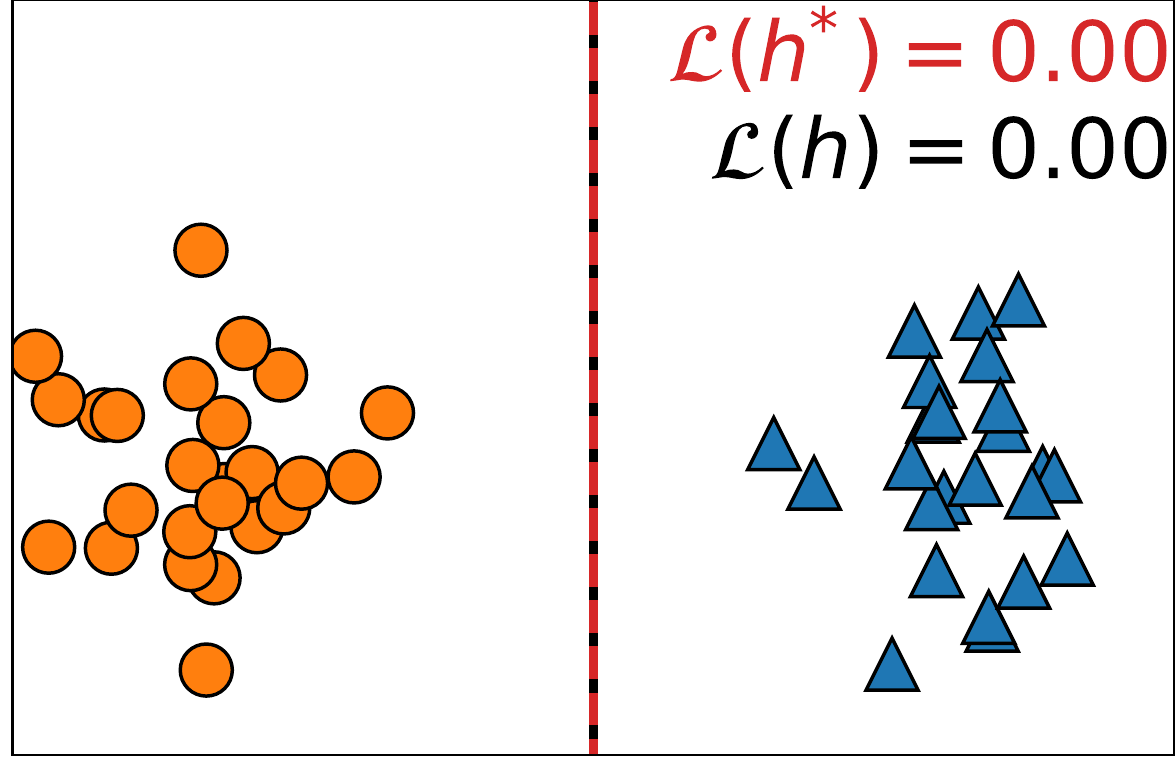}
    \subcaption{Before drift}
    \end{minipage}
    \hfill
    \begin{minipage}[b]{0.31\textwidth}
    \centering
    \includegraphics[width=\textwidth]{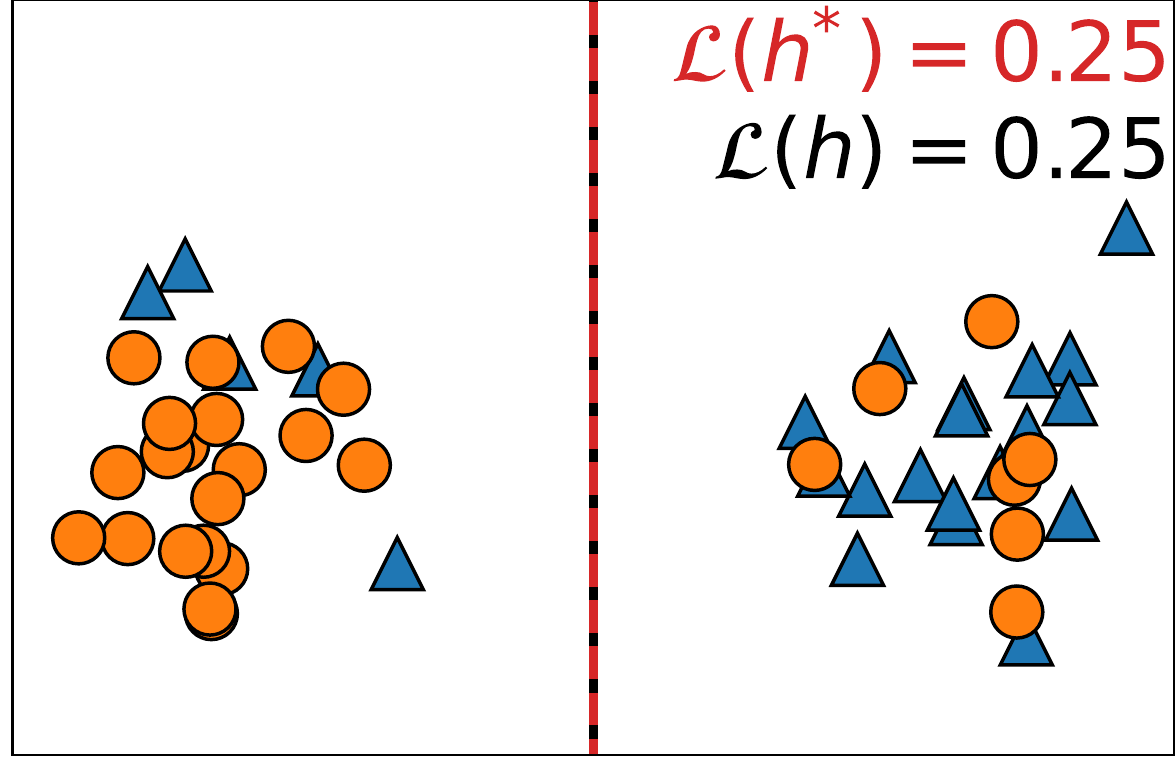}
    \subcaption{First drift: Add noise}
    \end{minipage}
    \hfill
    \begin{minipage}[b]{0.31\textwidth} 
    \centering
    \includegraphics[width=\textwidth]{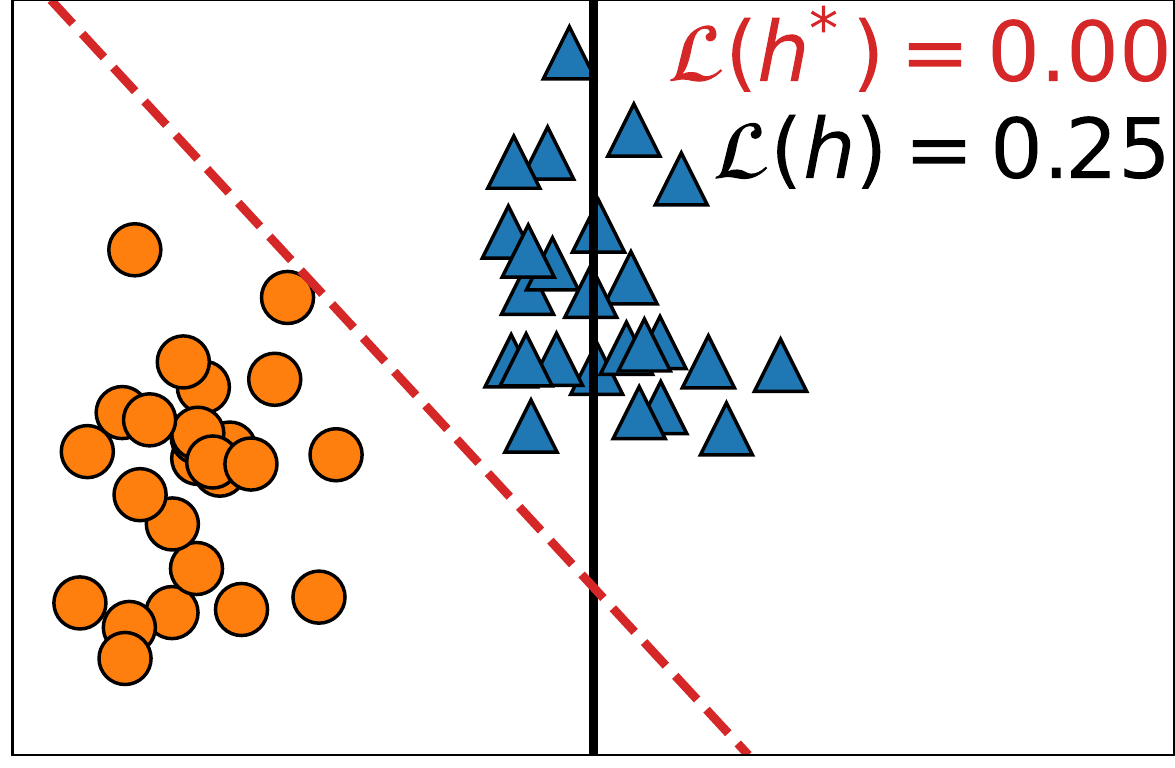}
    \subcaption{Second drift: Movement}
    \end{minipage}
    \hfill
    \caption{Effect of drift on model loss of a fixed and optimal model. Graphic shows fixed model $h$ (black line), optimal model $h^*$ (red dashed line), and model losses.}
    \label{fig:ITTE}
\end{figure}

In practice, few drift-learning algorithms refer to Eq.~\eqref{eq:loss_decompose}. Instead, a comparison of the current and historical loss is estimated using the ITTE scheme. This is commonly used to detect drift and, more generally, to evaluate the model ~\cite{asurveyonconceptdriftadaption}. However, this procedure is not flawless as can be seen in Figure~\ref{fig:ITTE}: The ITTE of a fixed model can change without a change of the optimal model and vice versa. 
Based on these insights, we aim for a better understanding usage of stream learning algorithms in the context of drift and novel techniques derived thereof, answering the following questions:
\begin{enumerate}
    \item\label{Q:2} How are model changes related to different types (real/virtual) of drift? 
    \item\label{Q:4} What is the relation of optimal models and the output of learning algorithms on different time windows? When to retrain the model?
    \item\label{Q:5} How are changes of the optimal model mirrored in changes of the ITTE? 
\end{enumerate}

\section{Theoretical Analysis}
\label{sec:Theo}
To 
answer these research questions we propose four formal definitions, each reflecting a different aspect and point of view of drift. We then compare those definitions, show formal implications, and provide counterexamples in case of differences, in order to provide the desired answers. We summarize our findings in Figure~\ref{fig:defs}, displaying different types of drift definitions and their implications.

We will refer to the types of drift that affect models as \emph{model drift}.
It is a  generalization of the notion of model drift in the work~\cite{DAWIDD,ContTime}, which is based on the comparison of the distribution 
for two different time windows, i.e.,\ $\D_{W_1} \neq \D_{W_2}$. We extend this idea  to incorporate model and loss-specific properties. 

\begin{figure}[t]
    \centering
    \begin{align*}
        \xymatrix{
        & & \text{real drift} \ar@{<=>}[rd]|{(4)} \ar@{=>}[lld] & \\
        \text{drift} & \text{$\ell$-model drift} \ar@{=>}[l]& \text{$A$-model drift} \ar@{<=>}[l]|{(1+3)} \ar@{=>}[d]|{(1)} & \text{$\H$-model drift} \ar@{=>}[l]|{(1)}\\
         & & \text{weak $\H$-model drift}  \ar@{<=>}[ur]|{(2)} \ar@{=>}[llu] \ar@{<=>}[ul]|{(3)}
        }
    \end{align*}
    \caption{Definitions and implications. Numbers indicate needed assumptions: $(1)$ $A$ is consistent, $(2)$ loss uniquely determines model, $(3)$ optimal loss is unchanged, $(4)$ universal hypothesis class of probabilistic models with non-regularized loss.}
    \label{fig:defs}
\end{figure}
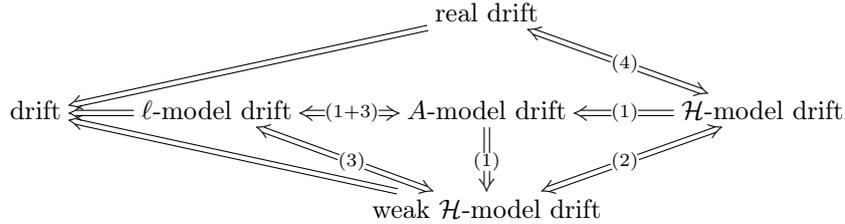

\subsection{Model Drift as Inconsistency of Optimal Models}
\label{sec:ModelDrift}

The concept of model drift can be considered from two points of view: different training results (see Section~\ref{sec:ChangeOfModel}) and inconsistency of optimal models. 
We deal with the latter notion first. Using loss as a proxy for performance, we consider that a model performs well if it has a loss comparable to the minimal achievable loss. 
We refer to this as \emph{hypothetical-} or \emph{$\H$-model drift}, which is defined as follows:

\begin{definition}
\label{def:H_model_drift}
Let $\H$ be a hypothesis class, $\ell$ a loss function on $\H$, and $\D_t$ be a drift process. 
We say that $\D_t$ has \emph{(strong) $\H$-model drift} iff there exist time windows without a common well-performing model, i.e.,\ there exist measurable $\P_T$ non-null sets $W_1,W_2 \subset \T$ and $C > 0$ such that for every $h \in \H$ 
either $\loss_{W_1}(h) > \inf_{h^* \in \H} \loss_{W_1}(h^*) + C$ or $\loss_{W_2}(h) > \inf_{h^* \in \H} \loss_{W_2}(h^*) + C$.
We say that $\D_t$ has \emph{weak $\H$-model drift} iff some model stops being optimal, i.e.,\ for some time windows $W_1, W_2$ 
there exists a $C > 0$ such that for all $\varepsilon < C$ there is some $h \in \H$
such that $\loss_{W_1}(h) \leq \inf_{h^* \in \H} \loss_{W_1}(h^*) + \varepsilon$ and $\loss_{W_2}(h) > \inf_{h^* \in \H} \loss_{W_2}(h^*) + C$.
\end{definition}

When not specified, $\H$-model drift refers to strong $\H$-model drift.
The difference between strong and weak $\H$-model drift is that strong $\H$-model drift rules out the existence of a single model that performs well during both time windows, whereas weak $\H$-model drift only states that there exists at least one  model that performs well on one but not the other time window. Thus, strong $\H$-model drift implies that model adaptation is strictly necessary for optimum results, whereas 
the necessity of model adaptation for weak $\H$-model drift might depend on the specific choice of the model. 
Strong $\H$-model drift implies weak $\H$-model drift.
This raises the question under which circumstances the converse is also true. It turns out that loss functions inducing unique optima are sufficient:

\begin{lemma}
\label{lem:H_model_drift}
If $\D_t$ has $\H$-model drift for windows $W_1,W_2$, then it has weak $\H$-model drift on the same windows.
If the optimal model is uniquely determined by the loss, i.e., for all $h_i,h_i' \subset \H$ with $\loss_{W_1}(h_i), \loss_{W_1}(h_i') \xrightarrow{i \to \infty} \inf_{h^* \in \H} \loss_{W_1}(h^*)$ we have $\limsup_{i \to \infty} |\ell(h_i,(x,y)) - \ell(h_i',(x,y))| = 0$
for all $(x,y) \in \X \times \Y$ and $\ell$ bounded, then the converse is also true.
The additional assumption is necessary.
\end{lemma}
\begin{proof}
All proofs can be found in the appendix.
\end{proof}
\begin{toappendix}
\appproof{lem:H_model_drift}{
\newcommand{\Cw}{C_\text{w}}
\newcommand{\Cs}{C_\text{s}}
Denote by $\Cw$ and $\Cs$ the $C$ from the definition of weak and strong $\H$-model drift, respectively.

\textbf{``Strong $\Rightarrow$ Weak'': } Choose $\Cw := \Cs$ for the weak case as in the strong case. By definition of the infimum for every $\varepsilon > 0$ there exists a $h \in \H$ such that $|\loss_{W_1}(h) - \inf_{h^* \in \H} \loss_{W_1}(h^*)| < \varepsilon$ and thus $\loss_{W_1}(h) \leq \inf_{h^* \in \H} \loss_{W_1}(h^*) + \varepsilon$. If we choose $\varepsilon < \Cw = \Cs$ then if follows by the definition of strong $\H$-model drift that $\loss_{W_2}(h) > \inf_{h^* \in \H} \loss_{W_2}(h^*)+\Cs$. 

\textbf{``Weak + unique $\Rightarrow$ Strong'': } 
Consider 
\begin{align*}
    \Cs = \frac{1}{3}\inf_{h^*_c \in \H} \left(\loss_{W_1}(h^*_c) - \inf_{h^*_1 \in \H}\loss_{W_1}(h^*_1) + \loss_{W_2}(h^*_c)- \inf_{h^*_2 \in \H}\loss_{W_2}(h^*_2)\right).
\end{align*}
Notice that $\Cs \geq 0$. If $\Cs \neq 0$ then $\D_t$ has strong $\H$-model drift, as otherwise there has to exist a model $h \in \H$ such that $\loss_{W_i}(h) \leq \inf_{h^* \in \H} \loss_{W_i}(h^*) + \Cs$ for all choices $i=1,2$ which implies
\begin{align*}
    \inf_{h^* \in \H} \loss_{W_1}(h^*) + \loss_{W_2}(h^*) 
    &\leq\loss_{W_1}(h)+\loss_{W_2}(h) 
    \\&\leq \inf_{h^* \in \H} \loss_{W_1}(h^*) + \inf_{h^* \in \H} \loss_{W_2}(h^*) + 2\Cs
    \\&\overset{\Cs > 0}{<} \inf_{h^* \in \H} \loss_{W_1}(h^*) + \inf_{h^* \in \H} \loss_{W_2}(h^*) + 3\Cs
    \\&= \inf_{h^* \in \H} \loss_{W_1}(h^*) + \loss_{W_2}(h^*)
\end{align*}
which is a contradiction. It remains to show that $\Cs > 0$.
Assume $\Cs = 0$.
By the definition of weak $\H$-model drift there exists a sequence $h_i \subset \H$ such that $\loss_{W_1}(h_i) \leq \inf_{h^* \in \H} \loss_{W_1}(h^*) + 2^{-i}\Cw$ and $\loss_{W_1}(h_i) > \inf_{h^* \in \H} \loss_{W_2}(h^*) + \Cw$. Let $h_i' \subset \H$ be a sequence with $\loss_{W_1}(h_i') \to \inf_{h^ * \in \H} \loss_{W_1}(h^*)$ and $\loss_{W_2}(h_i') \leq \inf_{h^ * \in \H} \loss_{W_2}(h^*) + \Cw/3$, which exists since $\Cs = 0$. By the uniqueness property $|\ell(h_i,(x,y)) - \ell(h_i',(x,y))| \to 0$ point wise and thus by Fatou's lemma $\int |\ell(h_i,(x,y)) - \ell(h_i',(x,y))| \d \mu \to 0$ for any finite measure $\mu$ on $\X \times \Y$, as $\ell$ is bounded. In particular, using Jensen's inequality, there exists a $n$ such that $|\loss_{W_2}(h_n) - \loss_{W_2}(h_n')| \leq \int |\ell(h_n,(x,y)) - \ell(h_n',(x,y))| \d \D_{W_2} < \Cw/3$. Now it follows
\begin{align*}
    \inf_{h^* \in \H} \loss_{W_2}(h^*) + \Cw 
    &< \loss_{W_2}(h_n) 
    \\&\leq \underbrace{|\loss_{W_2}(h_n) - \loss_{W_2}(h_n')|}_{\leq \Cw/3} + \underbrace{\loss_{W_2}(h_n')}_{\leq \inf_{h^* \in \H} \loss_{W_2}(h^*) + \Cw/3}
    \\&\leq \inf_{h^* \in \H} \loss_{W_2}(h^*) + 2\Cw/3 .
\end{align*}
This is a contradiction. 

\textbf{``Weak $\not\Rightarrow$ Strong'': } 
Let $\T = \{0,1\},\; W_i = \{i\}$, $\Y = \{0,1\}$, $\X = \{1,2,3\}$, $\H = \{ \1[\cdot \geq 2], \1[\cdot \neq 2] \}$, with 0-1-loss. $\D_0 = (\delta_{(1,0)} + \delta_{(2,1)})/2, \D_1 = (\delta_{(1,0)} + \delta_{(2,1)} + \delta_{(3,1)})/3$. Then both hypothesis have loss 0 at $T = 0$ but different loss at $T = 1$. Thus, there is no $\H$-model drift, but weak $\H$-model drift if we start with $h = \1[\cdot \neq 2]$.
}
\end{toappendix}

Notice that the uniqueness criterion becomes particular intuitive 
for functions $h : \X \to \Y$ and
the loss is induced by a metric, i.e., $\ell(h,(x,y)) = d(h(x),y)$, in which case we can bound $|\ell(h,(x,y)) - \ell(h',(x,y))| \leq d(h(x),h'(x))$. 
Thus, 
the criterion requires models with little variance to ensure that the notions of strong and weak $\H$-model drift coincide. This can
be achieved by a regularization term such as limiting the weight norm. 
As an immediate consequence we have:
\begin{corollary}
\label{cor:rest}
For $k$-nearest neighbor, RBF-networks, and decision tree virtual drift cannot cause $\H$-model drift, i.e., we do not have to clear the training window. For SVMs and linear regression based on the mean squared error virtual drift can cause $\H$-model drift, i.e., we may have to clear the training window.
\end{corollary}
\begin{toappendix}
\appproof{cor:rest}{
The first part follows from the fact that all considered models are dense subsets of the respective $L^2$ spaces for any finite base measure. See also Theorem~\ref{thm:virtual_drift}. 

For the second part consider the XOR classification problem, i.e., $Y \mid X = \delta_{\text{sign}(X_1 \cdot X_2 )}$, and only show one row per timepoint, i.e., for $\T = \{1,2\}$ set $\D_1(X) = \mathcal{U}([-1,0] \times [-1,1])$ and $\D_2(X) = \mathcal{U}([0,1] \times [-1,1])$, where $\mathcal{U}$ denotes the uniform distribution. Then there is no real drift but model drift since we have to flip the decision boundary as no linear model can learn the XOR problem.
}
\end{toappendix}

Obviously, (weak) $\H$-model drift implies drift because if there is no change of the loss, i.e., $\loss_t(h) = \loss_s(h)$ for all $h \in \H,\; s,t \in \T$, there cannot be $\H$-model drift. The converse is not so clear. We address this question in the following, targeting
real  drift. 

\begin{theorem}
\label{thm:virtual_drift}
Let $\Y = \{0,1\}$, $\T = [0,1]$, and $\X = \R^d$
. Let $\D_t$ be a drift process, $\H$ be a hypothesis class of probabilistic, binary classifiers, i.e., maps $h : \X \to [0,1]$, with MSE-loss, i.e., $\ell(h,(x,y)) = (h(x)-y)^2$, and assume that $\H$ is universal, i.e., dense in the compactly supported continuous functions $C_c(\X)$. 
Then, $\D_t$ has real drift if and only if $\D_t$ has $\H$-model drift.
\end{theorem}
\begin{toappendix}
\appproof{thm:virtual_drift}{

\textbf{``$\H$-model drift $\Rightarrow$ virtual drift'': } 
Assume there is no virtual drift, then $\E[Y\mid X] = \E[Y\mid X,T]$. Thus, for any time window $W \subset \T$ we have 
\begin{align*}
    \E[Y\mid X, T \in W] &\overset{\{T \in W\} \in \sigma(T)}{=} \E[\E[Y\mid X,T]\mid X, T \in W] \\&\overset{\text{Assum.}}{=} \E[\E[Y\mid X]\mid X, T \in W] \\&= \E[Y \mid X].
\end{align*}
Thus, for any window $W$ it follows that $\E[Y \mid X]$ is the optimal approximation that depends on $X$ only. In particular, we have
\begin{align*}
    \E[\ell_\MSE(h,(X,Y)) \mid T \in W] &=\quad \E[(h(X)-\E[Y\mid X])^2\mid T \in W] \\&\quad+ \E[\Var(Y\mid X,Z)\mid T \in W],
\end{align*}
whereby the second term on the right hand side does not depend on $h$. On the other hand, since $\X$ is a locally compact Borel space with countable basis, we have that all measures on $\X$ are regular and that the compactly supported continuous functions are dense in $L^p,\;p\in[1,\infty)$, which implies that $\H$ is also dense in $L^2$ and therefore we can approximate $\E[Y\mid X]$ arbitrarily well by functions in $\H$. But $\E[Y\mid X]$ does not depend on $T$ and therefore there cannot be $\H$-model drift. This is a contradiction.

\textbf{``virtual drift $\Rightarrow$ $\H$-model drift'': } 
Since $Y$ is Bernoulli and therefore uniquely determined by its mean and we have that $\E[(\E[Y\mid X] - \E[Y\mid X,T])^2] > 0$. Since $\X$ are second countable and therefore separable, and $\T$ is separable, so are $L^2(\X)$, $L^2(\T)$ and thus
the span of the simple functions $f(x)g(t) \in L^2(\X \times \T) \cong L^2(\X) \otimes L^2(\T)$ form a dense subset. Hence, there has to exist $f$ and $g$ such that $\E[ (\E[Y\mid X] - \E[Y\mid X,T]) \: \cdot \: f(X)g(T) ] \neq 0$. Since $\E[Y\mid X]$ is the best approximation of $\E[Y\mid X,T]$ in $L^2$ we have $\E[g(T)] = 0$. Set $W_1 = \{g > 0\}, W_2 = \{ g < 0\}$ then $\E[(\E[Y \mid X] + (-1)^ic f(X) - \E[Y \mid X,T])^2\mid T \in W_i] < \E[(\E[Y \mid X] - \E[Y \mid X,T])^2 \mid T \in W_i]$. Thus, there is $\H$-model drift for the windows $W_1,W_2$.

}
\end{toappendix}

This theorem includes crucial ingredients which are necessary to guarantee the result. 
As an example,
the model class has to be very flexible, i.e., universal, to adapt to arbitrary drift, and the loss function must enable such adaptation.

So far we considered the change of decision boundaries through the lens of models, disregarding how they are achieved.
We will take on a more practical point of view by considering models as an output of training algorithms applied to windows in the next section.

\subsection{Model Drift as Time Dependent Training Result}
\label{sec:ChangeOfModel}

Another way to consider the problem of model drift is to consider the output of a training algorithm. This idea leads to the second point of view:
drift manifests itself as 
the fact that 
the model obtained by training on data from one time point differs significantly from the model trained on data of another time point.
We  will refer to this notion as \emph{algorithmic-} or \emph{$A$-model drift}.
It answers the question of whether replacing a model trained on past data (drawn during $W_1$) with a model trained on new data (drawn during $W_2$) improves performance.
Using loss as a proxy we obtain the following definition:

\begin{definition}
\label{def:A_model_drift}
Let $\H$ be a hypothesis class, $\ell$ a loss function on $\H$, and $\D_t$ be a drift process. 
For a training algorithm $A$ we say that $\D_t$ has  \emph{$A$-model drift} iff model adaptation yields  a significant 
increase in performance with a high probability,
i.e.,\ there exist time windows $W_1,W_2$ such that for all $\delta > 0$ there exists a $C > 0$ and numbers $N_1$ and $N_2$ such that with probability at least $1-\delta$ over all samples $S_1$ and $S_2$ drawn from $\D_{W_1}$ and $\D_{W_2}$ of size at least $N_1$ and $N_2$, respectively, it holds $\loss_{W_2}(A(S_1)) > \loss_{W_2}(A(S_2)) + C$.
\end{definition}

Note that we do not specify how the algorithm processes the data, thus we also capture updating procedures. Indeed, removal of old data points, e.g., $W_1 = \{t_0,\dots,t_1,\dots,t_2\}, W_2 = \{t_1,\dots,t_2\}$, is a relevant instantiation of this setup.
Unlike $\H$-model drift which is concerned with consistency, it focuses on model change. The following theorem provides a connection between those notions. 

\begin{theorem}
\label{thm:theo_to_prac}
Let $\D_t$ be a drift process, $\H$ a hypothesis class with loss $\ell$ and learning algorithm $A$. Consider the following statement with respect to the same time windows $W_1$ and $W_2$: 
\begin{inparaenum}[(i)]
    \item $\D_t$ has $\H$-model drift for windows $W_1,W_2$.
   \item $\D_t$ has $A$-model drift for windows $W_1,W_2$.
   \item  $\D_t$ has weak $\H$-model drift for windows $W_1,W_2$.
\end{inparaenum}
If $A$ is a consistent training algorithm, i.e., for sufficiently large samples we obtain arbitrarily good approximations of the optimal model~\cite[Definition 7.8]{shalev2014understanding}, then 
$(i) \Rightarrow (ii) \Rightarrow (iii)$ holds. 
In particular, 
if we additionally assume that the optimal model is uniquely determined by the loss (see Lemma~\ref{lem:H_model_drift}) then all three statements are equivalent. 
If $A$ is not consistent, then none of the implications hold.
\end{theorem}
\begin{toappendix}
\appproof{thm:theo_to_prac}{
\newcommand{\Cw}{C_\text{w}}
\newcommand{\Cs}{C_\text{s}}
\newcommand{\Ca}{C_\text{A}}
Denote by $\Cw$, $\Cs$, and $\Ca$ the $C$ from the definition of weak, strong $\H$-model drift, and $A$-model drift, respectively.

\textbf{``Strong $\H$-model drift $\Rightarrow$ $A$-model drift'': } Choose $\Ca = \Cs/2$ and $h_1,h_2 \in \H$ such that $\loss_{W_i}(h_i) \leq \inf_{h^* \in \H} \loss_{W_i}(h^*) + \Ca/2$. Now, select $N_1,N_2$ according to $h_1,h_2$ and $\D_{W_1},\D_{W_2}$, respectively, such that the events $E_i = \{\loss_{W_i}(A(S_i)) \leq \loss_{W_i}(h_i) + \Ca/2\},\; i=1,2$ occur with
$\P[E_i] \geq 1-\delta/3,\; i=1,2$,
which is possible since $A$ is consistent. 
Thus, for all $\omega \in E_1 \cap E_2$ it holds
\begin{align*}
            \loss_{W_i}(A(S_i)) 
      &\leq \loss_{W_i}(h_i) + \Ca/2
    \\&\leq \inf_{h^* \in \H} \loss_{W_i}(h^*) + \Ca
    \\&\leq \inf_{h^* \in \H} \loss_{W_i}(h^*) + \Cs. \\
    \Rightarrow \loss_{W_2}(A(S_1)) &\overset{\textbf{!}}{>} \inf_{h^* \in \H} \loss_{W_2}(h^*) + \Cs \\&\geq \loss_{W_2}(A(S_2)) + \Ca,
\end{align*}
where \textbf{!} holds because $\loss_{W_1}(A(S_1)) \leq \inf_{h^* \in \H} \loss_{W_1}(h^*) + \Cs$ and there is strong $\H$-model drift, thus, $\loss_{W_1}(A(S_1)) > \inf_{h^* \in \H} \loss_{W_2}(h^*)+\Cs$.
This occurs with probability
\begin{align*}
    \P[E_1 \cap E_2] &= 1-\P[E_1^C \cup E_2^C] \\&\geq 1-(\P[E_1^C] + \P[E_2^C]) \\&\geq 1-2\delta/3 > 1-\delta.
\end{align*}

\textbf{``$A$-model drift $\Rightarrow$ weak $\H$-model drift'': }Let $N_1,N_2,\Ca$ be from the definition of $A$-model drift for $\delta = 1/3$.
Choose $\Cw = \Ca$. Let $\varepsilon > 0$ and $h_1 \in \H$ such that $\loss_{W_1}(h_1) \leq \inf_{h^*} \loss_{W_1}(h^*) + \varepsilon/2$. Since $A$ is consistent there exist $N_1'$ such that $\loss_{W_1}(A(S_1)) \leq \loss_{W_1}(h_1) + \varepsilon/2$ with probability at least $1-\delta$ over all choices of $S_1$ of size at least $N_1'$. Then with probability at least $1-2\delta = 1/3$ over all choices of samples $S_1,S_2$ of size at least $\max\{N_1,N_1'\}$ and $N_2$, respectively, it holds
\begin{align*}
    \loss_{W_1}(A(S_1)) &\leq \loss_{W_1}(h_1) + \varepsilon/2 
    \leq \inf_{h^* \in \H}\loss_{W_1}(h^*) + \varepsilon & \text{and} \\
    \loss_{W_2}(A(S_1)) &> \loss_{W_2}(A(S_2)) + \Ca 
    \geq \inf_{h^* \in \H} \loss_{W_2}(h^*) + \Cw.
\end{align*}
In particular, there exists an $\omega$ for which both statements hold and we may choose $h = A(S_1(\omega))$.

\textbf{``$A$ not consistent. Strong $\H$-model drift and no $A$-model drift'': }
$\T = \{1,2\},\;W_i=\{i\}$, $\Y = \{0,1\}$, $\X = \R$, $\H =  \{\1[\cdot > \theta] \mid \theta \in \X\}$ with 0-1-loss. 
$\D_1 = (\delta_{(-1,0)} + \delta_{(0,0)} + \delta_{(1,1)})/3$, $ \D_2 = (\delta_{(-1,0)} + \delta_{(0,1)} + \delta_{(1,1)})/3$. Then the optimal models are $h_1 \in \{\1[\cdot > \theta] \mid \theta \in [0,1)\}$, $h_2 \in \{\1[\cdot > \theta] \mid \theta \in [-1,0)\}$. Thus, there is $\H$-model drift and it holds $\loss_j(h_i) = (1-\delta_{ij})/3$. 
However, for $A : S \mapsto h_2$ we have $\loss_1(A(S_2)) = \loss_1(h_2) = 1/3 > 0 = \loss_2(h_2) = \loss_2(A(S_2))$
so there is no $A$-model drift.

\textbf{``$A$ not consistent. $A$-model drift and no weak $\H$-model drift'': }
$\T = \{1,2\},\;W_i=\{i\}$, $\Y = \{0,1\}$, $\X = \R$, $\H =  \{\1[\cdot > \theta] \mid \theta \in \X\}$ with 0-1-loss. 
$\D_1 = (\delta_{(-1,0)} + \delta_{(0,1)})/2$, $\D_2 = (\delta_{(-1,0)} + \delta_{(0,1)} + \delta_{(1,1)})/3$. Then in both cases the optimal models are given by $\{\1[\cdot > \theta] \mid \theta \in [-1,0)\}$, so there is no weak $\H$-model drift as the sets of optimal models coincide. Let 
\begin{align*}
    A : S \mapsto \begin{cases}
    \1[\cdot > -1/2] & \text{ if } (1,1) \not\in S \\
    \1[\cdot > 1/2] & \text{ otherwise}
    \end{cases}.
\end{align*}
Then $\loss_2(A(S_1)) = 1/3$ and $\P[\loss_2(A(S_2)) = 0] > (2/3)^{-|S_2|}$, so there is $A$-model drift.
}
\end{toappendix}

The relevance of this result follows from the fact that it connects theoretically optimal models to those obtained from training data when learning with drift.
The result implies that
model adaption does not increase performance if there is no drift. Further, 
if the model is uniquely determined 
any algorithm 
will suffer from drift in the same situations. 
Formally, the following holds:
\begin{corollary}
\label{cor:algorithm_agnostic}
Let $\D_t$ be a drift process, $\H$ a hypothesis class with consistent learning algorithms $A$ and $B$. Assume that the optimal model is uniquely determined by the loss, then for windows $W_1, W_2$ $A$-model drift is present if and only if $B$-model drift is present.
\end{corollary}
\begin{toappendix}
\appproof{cor:algorithm_agnostic}{
Let $W_1,W_2$ be time windows for which $A$ suffers from drift. By Theorem~\ref{thm:theo_to_prac} there is strong $\H$-model drift between $W_1$ and $W_2$ and thus for $h = B(S)$ we either have $\loss_{W_1}(B(S)) > \inf_{h^* \in \H} \loss_{W_1}(h^*) + C$ or $\loss_{W_2}(B(S)) > \inf_{h^* \in \H} \loss_{W_2}(h^*) + C$. Since $B$ is consistent for sufficient large sample sizes one of the statements is false and therefore the other must be correct.
}
\end{toappendix}

Although the results regarding $A$-model drift give us relevant insight, they do not yet include one important aspect of practical settings:
$A$-model drift compares already trained models, yet training a new model for every possible time window is usually unfeasible. Due to this fact, many algorithms investigate incremental updates and refer to the ITTE as an indicator of model accuracy and concept drift~\cite{asurveyonconceptdriftadaption}. We will investigate the validity of this approach in the next section. 

\subsection{Interleaved Train-Test Error as Indicator for (Model) Drift}
\label{sec:ChangeOfLoss}

A common technique to detect concept drift is to relate it to the performance of a fixed model. In this setup a decrease in performance indicates drift. Using loss as a proxy for performance we obtain the notion of \emph{loss-} or \emph{$\ell$-model drift} which corresponds to the ITTE:

\begin{definition}
\label{def:ell_model_drift}
Let $\H$ be a hypothesis class, $\ell$ a loss function on $\H$, and $\D_t$ be a drift process. 
We say that $\D_t$ has \emph{$\ell$-model drift} iff the loss of an optimal model changes, i.e., for time windows $W_1, W_2$ 
there exists a $C > 0$ such that for all $\varepsilon < C$ there is some $h \in \H$
such that $\loss_{W_1}(h) \leq \inf_{h^* \in \H} \loss_{W_1}(h^*) + \varepsilon$ and $\loss_{W_2}(h) > \loss_{W_1}(h) + C$.
We say that the \emph{optimal loss is non-decreasing/non-increasing/constant} iff $\inf_{h^* \in \H} \loss_{W_1}(h^*) \leq/\geq/= \inf_{h^* \in \H} \loss_{W_2}(h^*)$ holds.
\end{definition}

It is easy to see that $\ell$-model drift implies drift, the connection to the other notions of model drift is not so obvious as a change of the difficulty of the learning problem does not imply a change of the optimal model or vice versa: an example is the setup of a binary classification and drift induced change of noise level (Figure~\ref{fig:ITTE}).  Assumptions regarding the minimal loss lead to the following result:

\begin{lemma}
\label{lem:ell_and_H_model_drift}
Assume the situation of Definition~\ref{def:ell_model_drift}.
For time windows $W_1,W_2$ it holds:
\begin{inparaenum}[(i)]
    \item  For non-decreasing optimal loss, weak $\H$-model drift implies $\ell$-model drift.
    \item  For non-increasing optimal loss, $\ell$-model drift implies weak $\H$-model drift. 
\end{inparaenum}
The additional assumption is necessary.
\end{lemma}
\begin{toappendix}
\appproof{lem:ell_and_H_model_drift}{
\newcommand{\Cw}{C_\text{w}}
\newcommand{\Cl}{C_\ell}
Denote by $\Cw$ and $\Cl$ the $C$ from the definition of weak $\H$-model drift and $\ell$-model drift, respectively.

\textbf{``(i)'': } 
Choose $\Cl = \Cw/2$ and
let $h \in \H$ such that $\loss_{W_1}(h) \leq \inf_{h^* \in \H} \loss_{W_1}(h^*) + \Cl$ and $\loss_{W_2}(h) > \loss_{W_2}(h) + \Cw$.
Then it holds 
\begin{align*}
    \loss_{W_2}(h) &> \inf_{h^* \in \H} \loss_{W_2}(h^*) + \Cw \\&\geq \inf_{h^* \in \H} \loss_{W_1}(h^*) + 2\Cl \\&\geq \loss_{W_1}(h) + \Cl.
\end{align*}

\textbf{``(ii)'': } 
Choose $\Cw = \Cl$ and 
let $h \in \H$ such that $\loss_{W_1}(h) \leq \inf_{h^* \in \H} \loss_{W_1}(h^*) + \Cw$ and $\loss_{W_2}(h) > \loss_{W_1}(h) + \Cw$. Then it holds
\begin{align*}
    \loss_{W_2}(h) &> \loss_{W_1}(h) + \Cw \\&\geq \inf_{h^* \in \H} \loss_{W_1}(h^*) + \Cw \\&\geq \inf_{h^* \in \H} \loss_{W_2}(h^*) + \Cw.
\end{align*}

\textbf{Counterexamples of the necessity of non-increasing/non-decreasing: } $\T = \{0,1\},\;W_i=\{i-1\}$, $\Y = \{0,1\}$, $\X = \R$, $\H =  \{\1[\cdot > \theta] \mid \theta \in \X\}$ with 0-1-loss. 

\textbf{(i):} $\D_0 = (\delta_{(0,0)} + \delta_{(1,1)})/2$, $\D_1 = (\delta_{(0,0)}, \delta_{(1,1)})/2 \cdot 2/3 + (\delta_{(0,1)}, \delta_{(1,0)})/2 \cdot 1/3$. Then $h = \1[\cdot > 1/2]$ has $\loss_0(h) = 0$ and $\loss_1(h) = 1/3$ and thus there is $\ell$-model drift, but obviously $h$ is still a best possible model, thus there is no weak $\H$-model drift.

\textbf{(ii):} $\D_0 = (\delta_{(-1,0)} + \delta_{(0,0)} + \delta_{(0,1)} + \delta_{(1,1)})/4$, $\D_1 = (\delta_{(-1,0)} + \delta_{(0,0)} + \delta_{(1/2,1)} + \delta_{(1,1)})/4$. Then $h = \1[\cdot > -1/2]$ is a best possible model for $T = 0$ and it holds $\loss_0(h) = \loss_1(h) = 1/2$, so there is no $\ell$-model drift, but for $T = 1$ we have $\loss_1(\1[\cdot > 1/4]) = 0$, so there is $\H$-model drift.
}
\end{toappendix}

As a direct consequence of this lemma and Theorem~\ref{thm:theo_to_prac}, we obtain 
a criterion that characterizes in which cases active methods based on the ITTE are optimal. Here, we do not require that the loss uniquely determines the model:

\begin{theorem}
\label{thm:justify_active}
Let $\D_t$ be a drift process and $\H$ be a hypothesis class with loss $\ell$.
Assume the optimal loss is constant. Then for time windows $W_1,W_2$ and any consistent learning algorithm $A$ it holds: $\D_t$ has $A$-model drift if and only if it has $\ell$-model drift 
with respect to $h = A(S_1)$, i.e., $\forall \delta > 0 \exists N > 0 \forall n > N : \P_{S \sim \D^n_{W_1}}[\loss_{W_2}(A(S)) > \loss_{W_1}(A(S)) + C] > 1-\delta$.
\end{theorem}
\begin{toappendix}
\appproof{thm:justify_active}{

Denote by $C_A$ and $C_\ell$ the $C$ from the definition of $A$-model drift and $\ell$-model drift, respectively.

\textbf{``$A$-model drift $\Rightarrow$ $\ell$-model drift'': } Considering the proof of Theorem~\ref{thm:theo_to_prac} for every $\omega$ for which $h = A(S(\omega))$ has weak $\H$-model drift, it also has $\ell$-model drift, by (considering the proof of) Lemma~\ref{lem:ell_and_H_model_drift}.

\textbf{``$\ell$-model drift $\Rightarrow$ $A$-model drift'': } 
Choose $h_1,h_2 \in \H$ such that $\loss_{W_i}(h_i) \leq \inf_{h^*\in \H} \loss_{W_i}(h^*) + C_\ell/4,\; i =1,2$. Choose $N_1,N_2$ according to $h_1,h_2$ and $\D_{W_1},\D_{W_2}$ such that with probability at least $1-\delta/4$ over all choices of $S_i$ of size at least $N_i$ it holds $\loss_{W_i}(A(S_i)) \leq \loss_{W_i}(h_i) + C_\ell/4$. Thus, $\loss_{W_i}(A(S_i)) \leq \inf_{h^* \in \H} \loss_{W_i}(h^*) + C_\ell/2$. Then it holds
\begin{align*}
    \loss_{W_2}(A(S_1)) &> \loss_{W_1}(A(S_1)) + C_\ell \\&\geq \inf_{h^* \in \H} \loss_{W_1}(h^*) + C_\ell \\&= \inf_{h^* \in \H} \loss_{W_2}(h^*) + C_\ell \\&\geq \loss_{W_2}(A(S_2)) + C_\ell/2,
\end{align*}
with probability at least $1-3\delta/4 > 1-\delta$ over all choices of $S_1,S_2$ with size at least $N_1,N_2$. Thus, by setting $C_A = C_\ell/2$ the statement follows.
}
\end{toappendix}

Notice that this result provides a theoretical justification for the common practice in active learning, to use drift detectors on the ITTE determining whether or not to retrain the model. The statement only holds if the optimal loss is constant -- otherwise, the ITTE is  misleading and can result in both false positive and false negative implications (see Figure~\ref{fig:ITTE}). 

\section{Empirical Evaluation}
\label{sec:Experiments}

In the following, we demonstrate our theoretical insights in experiments and quantify their effects. 
All results which are reported in the following are statistically significant (based on a $t$-test, $p < 0.001$).
All experiments are performed on the following standard synthetic benchmark datasets 
AGRAWAL~\cite{Agrawal1993DatabaseMA},
LED~\cite{asuncion2007uci},
MIXED~\cite{LearningWithDriftDetection},
RandomRBF~\cite{skmultiflow},
RandomTree~\cite{skmultiflow},
SEA~\cite{seadata},
Sine~\cite{LearningWithDriftDetection}, 
STAGGER~\cite{LearningWithDriftDetection}
and the following real-world benchmark datasets
``Electricity market prices''~(Elec)~\cite{electricitymarketdata}, 
``Forest Covertype''~(Forest)~\cite{forestcovertypedataset}, and 
``Nebraska Weather'' (Weather)~\cite{weatherdataset}.
To remove effects due to unknown drift in the real-world datasets, we apply a permutation scheme~\cite{ida2022}, and we induce real drift  by a label switch. As a result, all datasets have controlled real drift and no virtual drift. We induce virtual drift by segmenting the data space using a random decision tree.
For comparability, all problems are turned into binary classification tasks with class imbalance below $25\%$. 
This way we obtained $2 \times 2$ distributions with controlled drifting behavior, i.e., $\D_{ij}(X,Y) = \D_i(X) \D_j(Y|X), \; i,j \in \{0,1\}$.

\begin{figure}[t]
    \centering
    \begin{minipage}[b]{0.28\textwidth}
    \centering
    \includegraphics[width=\textwidth]{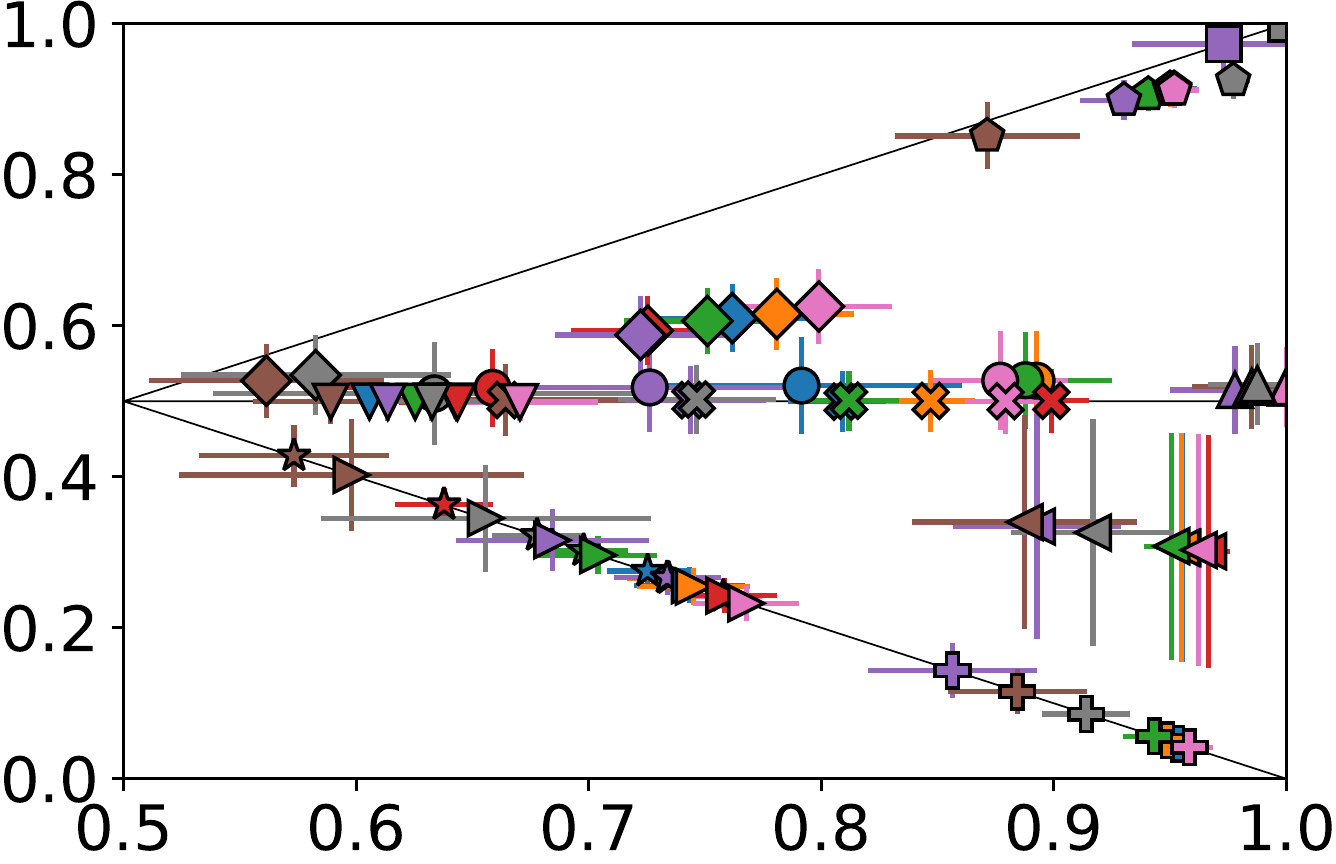}
    \subcaption{None vs. real drift}
    \end{minipage}
    \hfill
    \begin{minipage}[b]{0.28\textwidth}
    \centering
    \includegraphics[width=\textwidth]{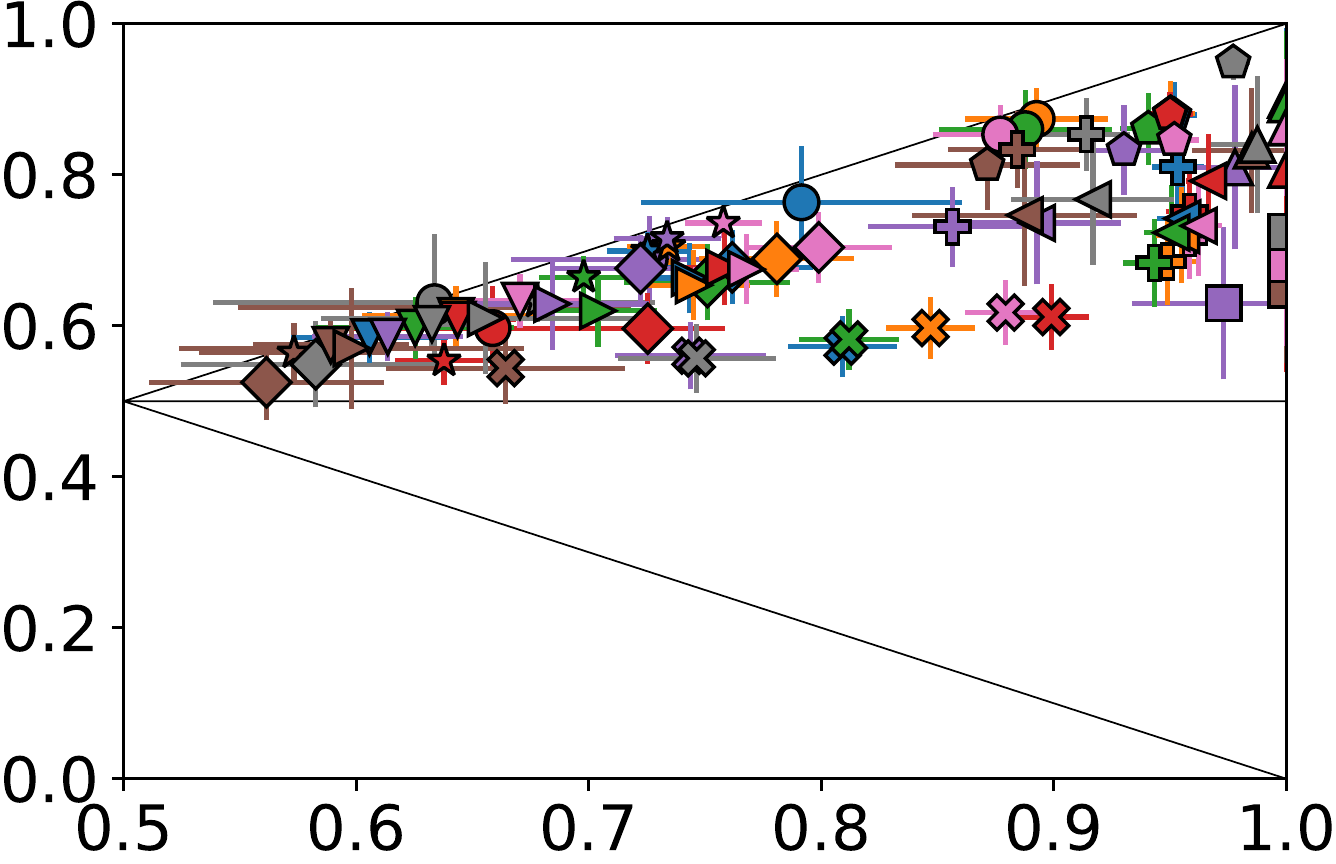}
    \subcaption{None vs. virtual drift}
    \end{minipage}
    \hfill
    \begin{minipage}[b]{0.28\textwidth} 
    \centering
    \includegraphics[width=\textwidth]{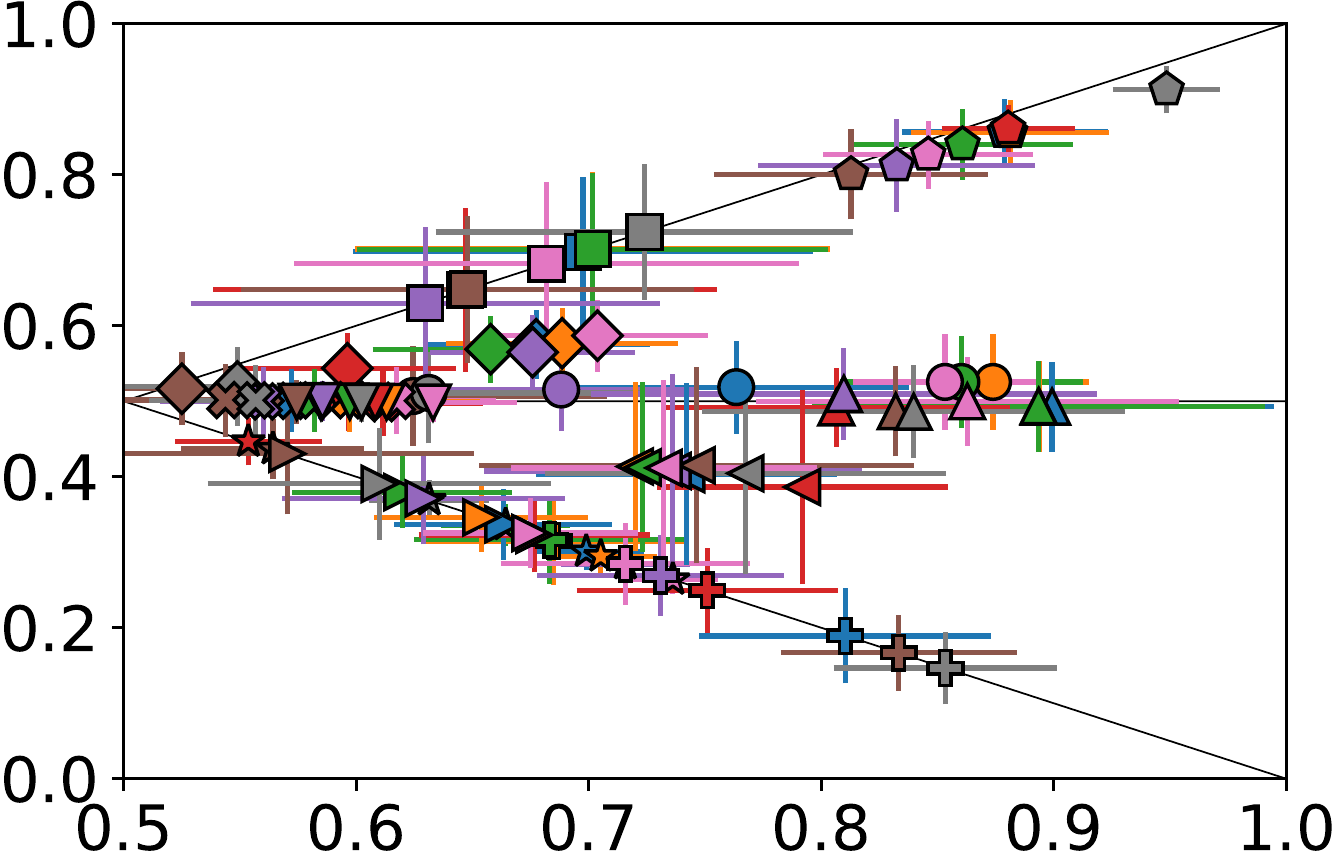}
    \subcaption{Virtual drift vs. both}
    \label{fig:exp1:c}
    \end{minipage}
    \hfill
    \begin{minipage}[b]{0.10\textwidth} 
    \centering
    \includegraphics[width=\textwidth,trim={0 0 7cm 0},clip]{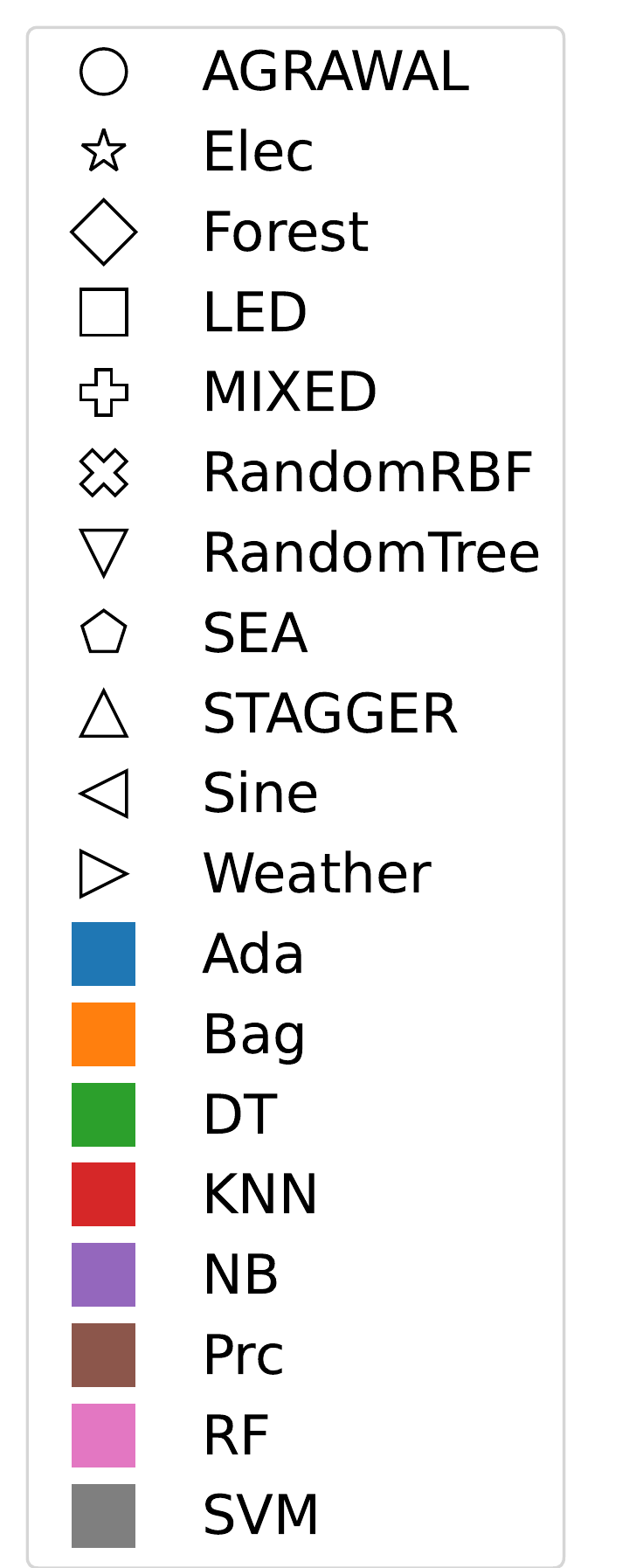}
    \end{minipage}
    \caption{Comparison of types of drift for different datasets (marker) and models (color). For the sake of clarity, error bars show $\frac{1}{2}$ of standard deviation. }
    \label{fig:exp1}
\end{figure}

To show the effect of real and virtual drift on classification accuracy, we draw train and test samples from those distributions which correspond to the time windows in Section~\ref{sec:Theo} and compute the train-test error of the following models: Decision Tree (DT), Random Forest (RF), $k$-Nearest Neighbour ($k$-NN), Bagging (Bag; with DT), AdaBoost (Ada; with DT), Gaussian Na\"ive Bayes (NB), Perceptron (Prc), and linear SVM (SVM)~\cite{scikit-learn}.
We repeated the experiment $1{,}000$ times. The results are shown in Figure~\ref{fig:exp1}. We found that real and virtual drift causes a significant decrease in accuracy compared to the non-drifting baseline for all models and datasets (except for Prc and SVM on AGRAWAL on virtual drift where the results are inconclusive). A combination of real and virtual drift decreased the accuracy even further if compared to the non-drifting baseline and virtual drift only.
These findings are in strong agreement with Theorem~\ref{thm:justify_active} and show that virtual drift can cause a significant decrease in accuracy although it is usually considered less relevant for the performance of a model.

\begin{figure}[t]
    \centering
    \begin{minipage}[b]{0.31\textwidth}
    \centering
    \includegraphics[width=\textwidth]{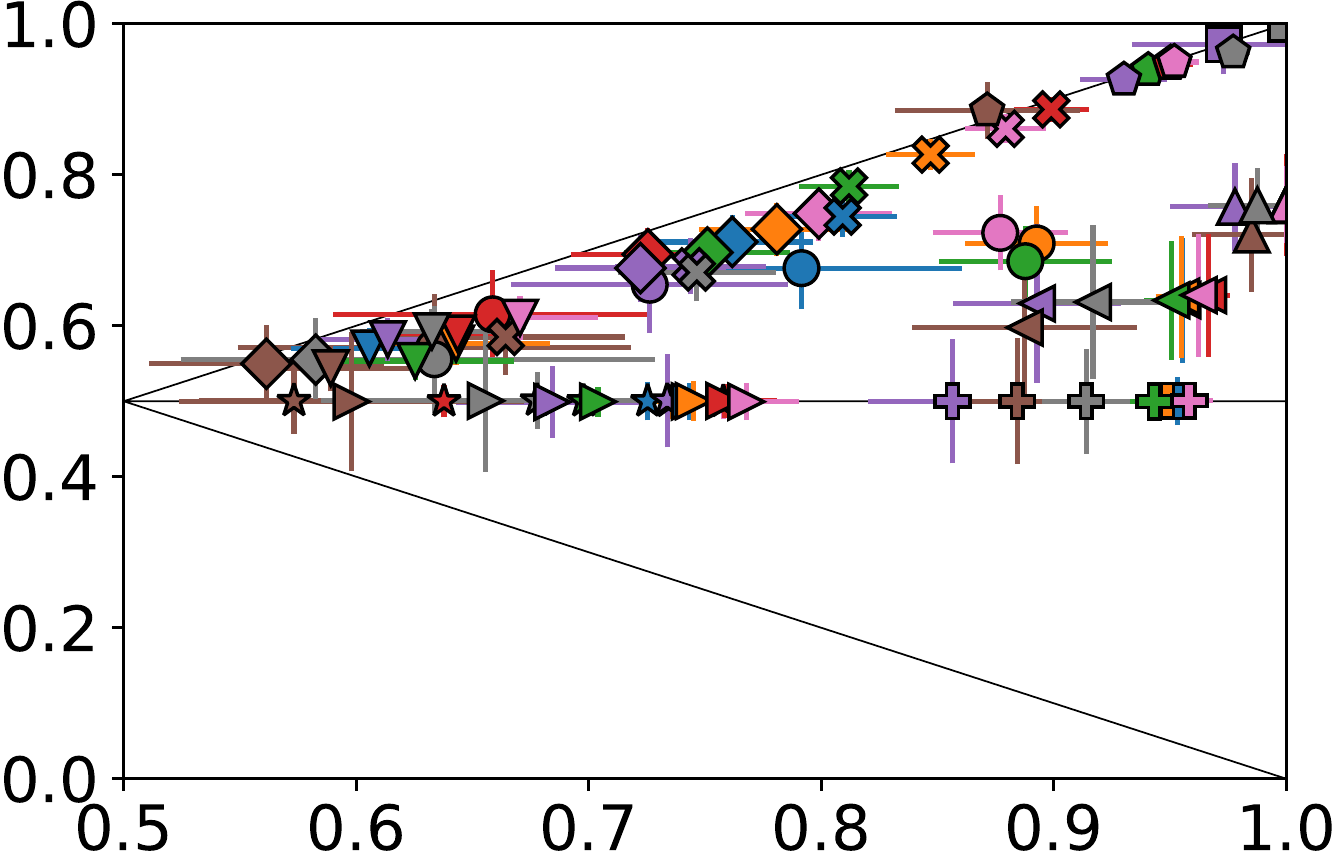}
    \subcaption{None vs. composed real drift}
    \end{minipage}
    \hfill
    \begin{minipage}[b]{0.31\textwidth}
    \centering
    \includegraphics[width=\textwidth]{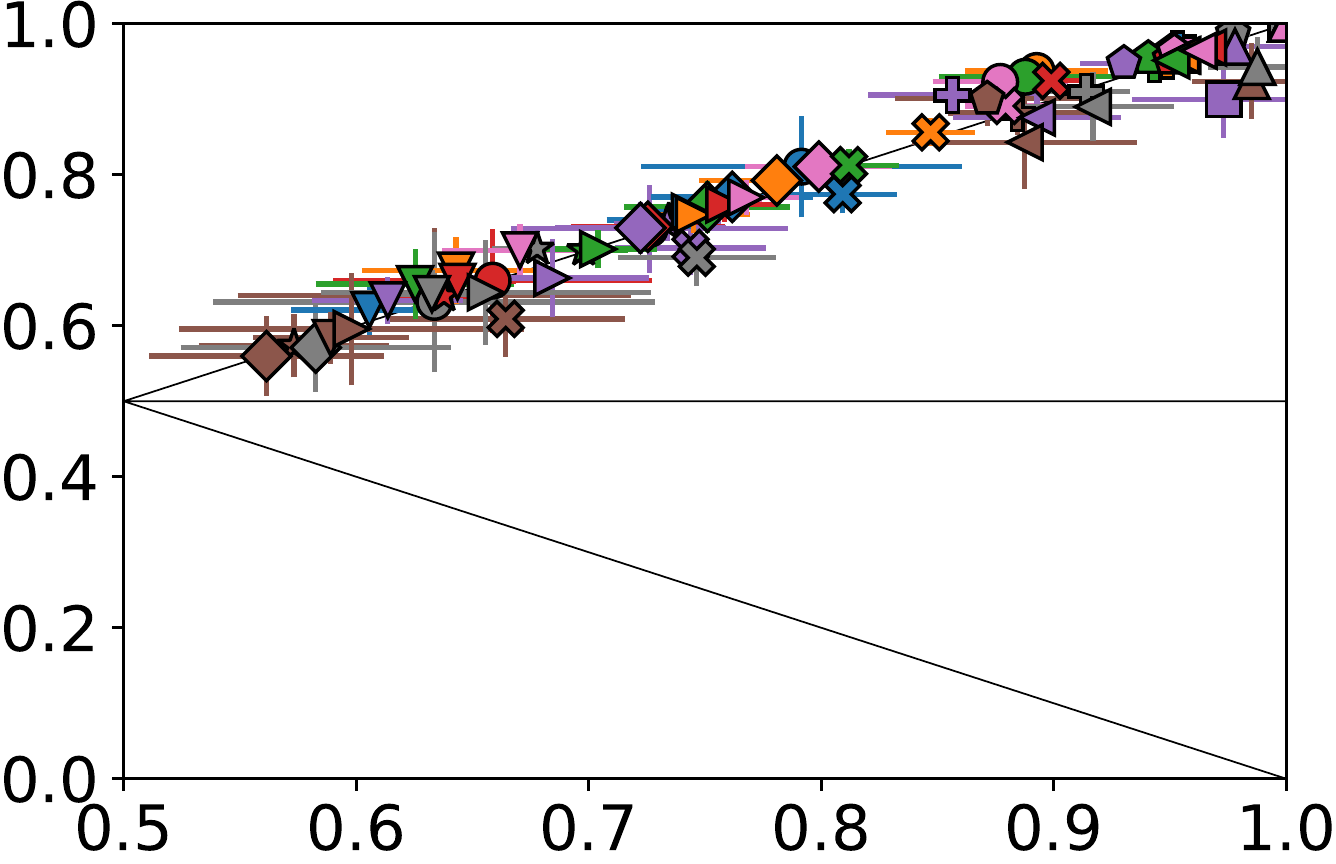}
    \subcaption{None vs. composed virtual drift}
    \end{minipage}
    \hfill
    \begin{minipage}[b]{0.31\textwidth} 
    \centering
    \includegraphics[width=\textwidth]{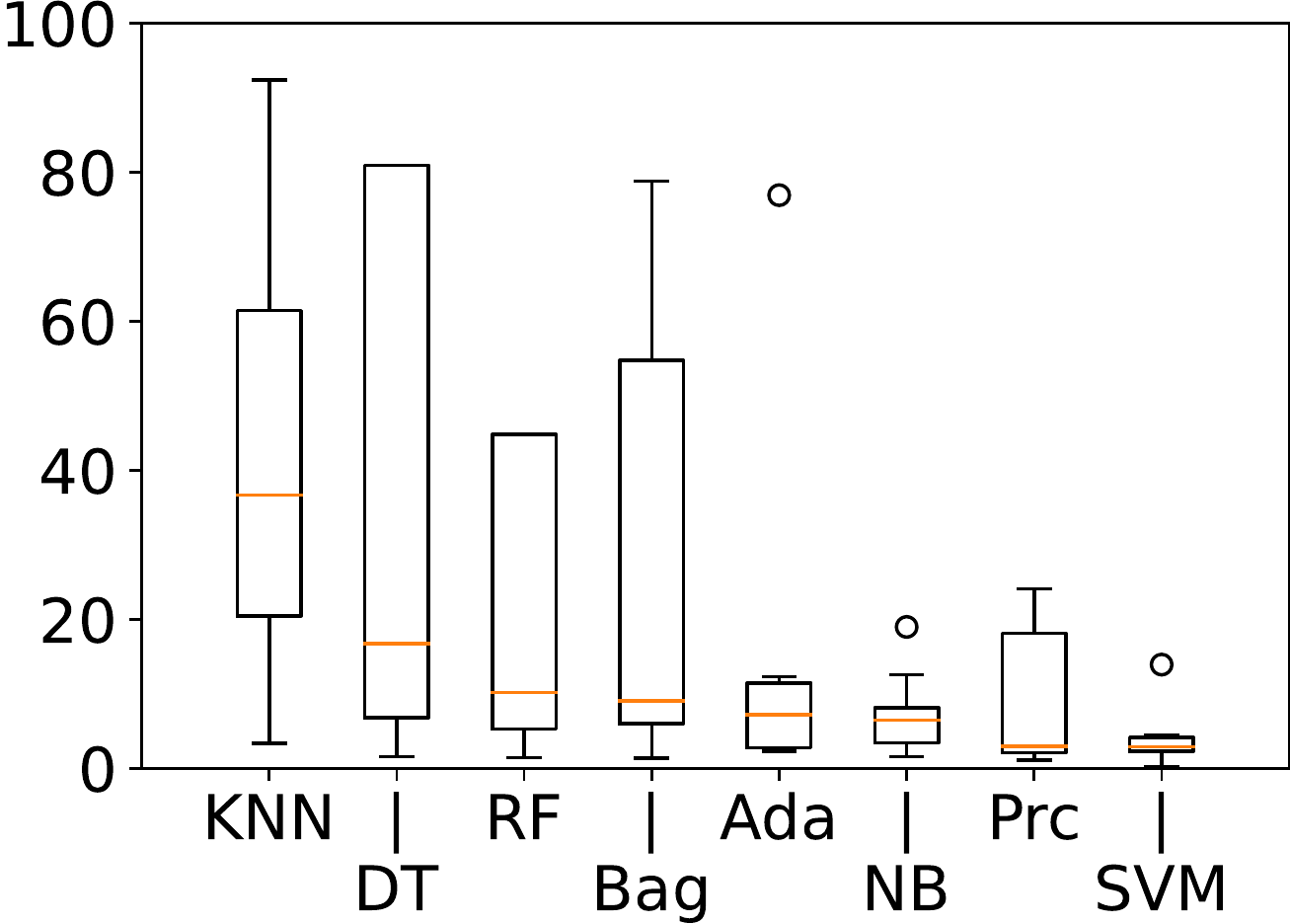}
    \subcaption{Usage of additional information}
    \label{fig:exp2:c}
    \end{minipage}
    \hfill
    \caption{Evaluation of composed windows. Plots (a) and (b) use the same color/marker scheme as Figure~\ref{fig:exp1}.}
    \label{fig:exp2}
\end{figure}

To evaluate the necessity to reset the training window after drift we combined two windows that differ in one drift type, i.e., virtual or real drift, and
proceed as before. An overview of the results is presented in Figure~\ref{fig:exp2}. As expected, the models trained on the composed windows outperform the ones trained on the non-composed samples (except for SVM on AGRAWAL with virtual drift where the results are inconclusive). 
In comparison to the non-drifting baseline, the composed real drift models are outperformed, and the composed virtual drift model are mainly inconclusive.
A further analysis of the latter scenario is presented in Figure~\ref{fig:exp2:c} (we consider $\left|(c-v)/(n-c)\right|$ which is a measure for the usage of additional information where $c$ is composed virtual, $v$ virtual, and $n$ no drift).
As can be seen, NB, Prc, and SVM do not profit, and DT and $k$-NN profit most, RF, Bag, and Ada profit moderately from the additional information. 
These findings are in strong agreement with Corollary~\ref{cor:rest} and Lemma~\ref{lem:H_model_drift} as they quantitatively show that more flexible models are better at handling virtual drift in the training window.

\section{Discussion and Conclusion}
\label{sec:Conclusion}

In this work, we considered the problem of online and stream learning with drift from a theoretical point of view. Our main results aim at the application of active methods that adapt to drift in data streams by mainly considering the ITTE. In contrast to many other works in this area, we focused on consistency and/or change of the decision boundary as indicated by models and loss functions. Furthermore, our approach also applies to semi- and unsupervised setups, e.g., clustering, dimensionality reduction, etc. More general notions of time, e.g., computational nodes as considered in federated learning, are also covered. To the best of our knowledge, it is the first of this kind.

\bibliographystyle{acm}
\bibliography{bib}

\end{document}